\documentclass[twocolumn]{IEEEtran}
\usepackage[utf8]{inputenc}
\usepackage{amsthm}
\usepackage{amsfonts}
\usepackage{import}
\usepackage{algorithm,algorithmic}
\usepackage{lipsum}
\usepackage{mcode}
\usepackage{graphicx}
\usepackage{multirow}
\usepackage{subfigure}
\usepackage{cite}
\usepackage[mathscr]{eucal}
\usepackage{relsize}
\usepackage{stmaryrd} 
\usepackage{bm}
\usepackage{amsmath, amsthm} 

\newcommand{\mat}[1]{\ensuremath{\boldsymbol{\mathbf{#1}}}}

\newcommand{\tensor}[1]{\ensuremath{\boldsymbol{\mathscr{{#1}}}}}

\newcommand{\hide}[1]{}

\newcommand{\tp}{\ensuremath{\tensor{P}}}

\newcommand{\tb}{\ensuremath{\tensor{B}}}

\newcommand{\tgk}{\ensuremath{\tensor{G}^{(k)}}}
\newcommand{\mg}{\ensuremath{\mat{G}}}

\newcommand{\tgh}{\ensuremath{\tilde{\tensor{G}}}}
\newcommand{\tgkh}{\ensuremath{\tilde{\tensor{G}}}^{(k)}}
\newcommand{\mgh}{\ensuremath{\tilde{\mat{G}}}}

\newcommand{\tm}{\ensuremath{\tensor{M}}}
\newcommand{\tmk}{\ensuremath{\tensor{M}^{k}}}

\newcommand{\tqk}{\ensuremath{\tensor{Q}^{k}}}
\newcommand{\tq}{\ensuremath{\tensor{Q} }}

\newcommand{\ts}{\ensuremath{\tensor{S}}}
\newcommand{\te}{\ensuremath{\tensor{E}}}

\newcommand{\mmk}{\ensuremath{\mat{M}_{[k]}^{k}}}

\newcommand{\matr}{\ensuremath{\mathbb{R}}}

\newcommand{\tx}{\ensuremath{\tensor{X}}}
\newcommand{\mx}{\ensuremath{\mat{X}}}
\newcommand{\mxk}{\ensuremath{\mat{X}_{[k]}}}

\newcommand{\txh}{\ensuremath{\tilde{\tensor{X}}}}

\newcommand{\mxhk}{\ensuremath{ \tilde{\mat{X}}_{[k]} }}

\newcommand{\ty}{\ensuremath{\tensor{Y}}}

\newcommand{\my}{\ensuremath{\mat{Y}}}

\theoremstyle{plain}
\newtheorem{definition}{Definition}

\theoremstyle{plain}

\newtheorem{theorem}{Theorem}

\title{Efficient Tensor Robust PCA under Hybrid Model of Tucker and Tensor Train}

\begin{document}
		\author{Yuning~Qiu, Guoxu~Zhou, Zhenhao~Huang, Qibin~Zhao, \IEEEmembership{Senior Member, IEEE}, \\ and Shengli~Xie, \IEEEmembership{Fellow, IEEE}
		\thanks{This work is supported in part by Natural Science Foundation of China under Grant 61673124, Grant 61903095, Grant 61727810, and Grant 61973090. (\emph{Corresponding authors: Guoxu Zhou.})}
		\thanks{Y. Qiu, G. Zhou, Z. Huang, Q. Zhao and S. Xie are with the School of Automation, Guangdong University of Technology, Guangzhou 510006, China (e-mail: yuning.qiu.gd@gmail.com, gx.zhou@gdut.edu.cn,zhhuang.gdut@qq.com,qibin.zhao@riken.jp, shlxie@gdut.edu.cn).}
		\thanks{Q. Zhao is also with the Center for Advanced Intelligence Project (AIP), RIKEN, Tokyo, 103-0027, Japan (e-mail: qibin.zhao@riken.jp).}
	}
	\markboth{Journal of \LaTeX\ Class Files,~Vol.~14, No.~8, August~2015}%
	{Shell \MakeLowercase{\textit{et al.}}: Bare Demo of IEEEtran.cls for IEEE Journals}
\maketitle
\begin{abstract}\\	
Tensor robust principal component analysis \mbox{(TRPCA)} is a fundamental model in machine learning and computer vision. Recently, tensor train (TT) decomposition has been verified effective to capture the global low-rank correlation for tensor recovery tasks. However, due to the large-scale tensor data in real-world applications, previous TRPCA models often suffer from high computational complexity. In this letter, we propose an efficient TRPCA under hybrid model of Tucker and TT. Specifically, in theory we reveal that TT nuclear norm (TTNN) of the original big tensor can be equivalently converted to that of a much smaller tensor via a Tucker compression format, thereby significantly reducing the computational cost of singular value decomposition (SVD). Numerical experiments on both synthetic and real-world tensor data verify the superiority of the proposed model. 
\end{abstract}


\section{Introduction}

Tensor decomposition is a fundamental tool for multi-way data analysis \cite{TamaraG.KoldaBrettW.Bader,spm-cichocki-2015,zhang2017low}. In real-world applications, tensor data are often corrupted with sparse outliers or gross noises  \cite{inoue2009robust,Zare2018}. 
For example, face images recorded in practical applications might contaminate gross corruptions due to illumination and occlusion noise. To alleviate this issue,  tensor robust principal component analysis \cite{zhang2014novel,Goldfarb2014,Ren2017,Bahri2019a}  (TRPCA) or robust tensor decomposition \cite{Gu2014a,Zhang2016a,li2013robust}  (RTD) were proposed to estimate the underlying low-rank and sparse components from their sum.
In the past decades, it has been shown that the low-rank and sparse components can be exactly recovered by solving the following minimization problem:
\begin{equation}
\min_{\tx, \ts}  \| \tx  \|_{\varoast} + \tau \| \ts \|_1 , \text{ s.t. } \ty = \tx + \ts, 
\end{equation}
where $ \| \cdot \| _{\varoast}$ and $\| \cdot\|_1$ denote tensor nuclear norm and $\ell_1$ norm, respectively,  and $\tau >0$ is a hyper-parameter balancing the low-rank and sparse components.

The essential element of the TRPCA problem is to capture the high-order low-rank structure. Unlike the matrix case, there is no unified tensor rank definition due to the complex multilinear structure. 
The most straightforward tensor rank is CANDECOMP/PARAFAC (CP) rank, which is defined on the smallest number of rank-one tensors \cite{TamaraG.KoldaBrettW.Bader}.  Nevertheless, both CP rank and its convex surrogate CP nuclear norm are NP-hard to compute \cite{10.1145/2512329,Friedland2017}. To alleviate such issue,  Zhao et al. \cite{Zhao2016}  proposed a variational Bayesian inference framework for CP rank determination and applied it to the TRPCA problem. 
Compared with CP rank, Tucker rank is  more flexible and interpretable since it explores the low-rank structure in all modes. 
Tucker rank is defined as a set of ranks of unfolding matrices associated with each mode \cite{tucker1966some}. Motivated by the convex surrogate of matrix rank, the sum of the nuclear norm (SNN) was adopted as a convex relaxation for Tucker rank \cite{JiLiu2010b}. The work \cite{Goldfarb2014} proposed an SNN-based TRPCA model.   Huang et al.  \cite{huang2015provable} presented the exact recovery guarantee for SNN-based TRPCA. 
Recently, TT rank-based models have achieved both theoretical and practical performance better than Tucker rank models in the field of tensor recovery applications.  
Compared with Tucker rank, TT rank was demonstrated to capture global correlation of tensor entries using the concept of von Neumann entropy theory \cite{Bengua2017a}.  Gong et al. \cite{9200561} showed the potential advantages of TT rank by investigating the relationship between Tucker decomposition and TT decomposition. 
Similar to SNN, Bengua et al. \cite{Bengua2017a} proposed  TT nuclear norm  (TTNN) using the sum of nuclear norm of TT unfolding matrix along with each mode. 
Yang et al. \cite{Yang2020a}  proposed a TRPCA model based on TTNN   and applied it to tensor denoising tasks.  


It is significantly important to have efficient optimization algorithms for TRPCA problem. This is especially challenging for large-scale and high-order tensor data. 
In this letter, we propose an efficient TT rank-based TRPCA model which equivalently converts the TTNN minimization problem of the original tensor to that of a much smaller tensor. Thus, the computational complexity of TTNN minimization problem can be significantly reduced. 
To summarize, we make the following contributions:
\begin{itemize}
	\item We demonstrate that  the minimization of TTNN on the original big tensor can be equivalently converted to a much smaller tensor under a Tucker compression format. 
	\item We propose an efficient TRPCA model and develop an effective alternating direction method of multipliers (ADMM) based optimization algorithm. 
	\item We finally show that the proposed model achieves more promising recovery performance and less running time than the state-of-the-art models on both synthetic and real-world tensor data.
\end{itemize}

 
\section{Notation and Preliminary}

\subsection{Notations}
We adopt the notations used in \cite{TamaraG.KoldaBrettW.Bader}. The set $\{ 1,2,\cdots,K\}$ is denoted as $[K]$.  A scalar is given by a standard lowercases or uppercases letter $x, X \in \matr$. A matrix is given by a boldface capital letter $\mx \in \matr^{d_1\times d_2}$. A tensor is given by calligraphic letter $\tx \in \matr^{d_1\times d_2\times  \cdots d_K}$. The $(i_1,i_2,\cdots,i_K)$ entry of tensor $\tx$ is given by $X(i_1,i_2,\cdots,i_K)$.  The standard mode-$k$ unfolding \cite{TamaraG.KoldaBrettW.Bader} of tensor $\tx \in \matr^{d_1 \times d_2 \times \cdots \times d_K}$ is given by $\mx_{(k)} \in \matr^{d_k \times d_1 \cdots d_{k-1} d_{k+1} \cdots d_{K}} $, and the corresponding \mbox{matrix-tensor} folding operation is given by $\mcode{sfold}_k(\mx_{(k)})$. Another mode-$k$ unfolding for TT decomposition \cite{Wen2008} is denoted as $\mx_{[k]} \in \matr^{d_1 \cdots d_{k-1} \times d_k \cdots d_{K} } $, and the corresponding matrix-tensor folding operation is denoted as $ \mcode{fold}_k ( \mx_{[k]} ) $. 

\subsection{Tensor Train Decomposition}
 \begin{definition} [TT decomposition \cite{Wen2008}]
 	The tensor train (TT) decomposition represents a $K$th-order tensor $\ty \in  \mathbb{R}^{d_1 \times \cdots \times d_K}$ by the sequence multilinear product over a set of third-order core tensors, i.e., $\tensor{Y} = \emph{\mcode{TT}}( \tensor{G}^{(1)}, \cdots, \tensor{G}^{(K)} )$, where $\tensor{G}^{(k)} \in \mathbb{R}^{r_{k-1} \times d_k \times r_{k}}$, $k \in [K]$, and $r_{K} = r_{0} = 1$. Element-wisely, it can be represented as
 	\begin{equation}
 		 {Y}({i_1, i_2, \cdots, i_K}) = \sum_{\upsilon_0, \cdots , \upsilon_K }^{r_0, \cdots, r_K} \prod_{k=1}^{K}  {G}^{(k)} ( \upsilon_{k-1},  i_k,\upsilon_{k} ).
 	\end{equation}
 	The size of cores, $r_k, k=1,\cdots, K-1$, denoted by a vector $[r_1, \cdots, r_{K-1}]$, is called {TT rank}. 
 \end{definition}


\begin{definition} [TT nuclear norm \cite{Bengua2017a}] \label{def-ttnn}
	The tensor train nuclear norm (TTNN) is defined by the weighted sum of nuclear norm along each unfolding matrix:
	\begin{equation}
		\| \ty \|_{\emph{ttnn}} := \sum_{k=1}^{K-1} \alpha_k \| \my_{[k]} \|_*,
	\end{equation}
where $\| \cdot \|_*$ is the matrix nuclear norm, $\alpha_k   > 0$ denotes the weight of mode-$k$ unfolding.
\end{definition}

\section{Efficient TRPCA under Hybrid Model of Tucker and TT}
\subsection{Fast TTNN Minimization under a Tucker Compression}

In the following theorem, we show that TT decomposition can be equivalently given in a Tucker compression format. 
\begin{theorem} \label{theorem-rank-property}
Let  $\tx \in \matr^{d_1\times d_2 \times \cdots \times d_K}$ be a $K$th-order tensor with TT rank $[r_1, \cdots r_{K-1} ]$, it can be formulated as the following Tucker decomposition format:
	\begin{equation}
		\tx = \txh \times_1 \mat{U}_1 \times_2 \cdots \times_K \mat{U}_K, \mat{U}_k \in St(d_k, R_k), 
	\end{equation}
where $\txh =\emph{\mcode{TT}}(\tgh^{(1)}, \tgh^{(2)},\cdots,\tgh^{(K)} )$ denotes the core tensor,  $ \tgkh \in \matr^{r_k \times R_k \times r_{k+1} }$, and $St(d_k, R_k):= \{ \mat{U} \in \matr^{d_k \times R_k},  \mat{U}^{\top} \mat{U} = \mat{I}_{R_k} \}$ is the Stiefel manifold. 
\end{theorem}
\begin{proof}
The proof can be completed by discussing in the following two cases. 

\emph{Case 1:} If $d_k > r_k r_{k+1}$,  we let $\mat{U}_k$ be the left singular matrix of mode-2 unfolding of  $\tgk$, i.e., $[\mat{U}_k, \bm{\Sigma}_k, \mat{V}_k^{\top}] = \mcode{SVD}(   \mg^{(k)}_{(2)}  ) $, and we let $\tgh^{(k)} = \mcode{sfold}_2(\bm{\Sigma}_k \mat{V}_k^\top)$. 

\emph{Case 2:} If $d_k \leq r_k r_{k+1}$, we let $\mat{U}_k$ be the identity matrix, i.e., $\mat{U}_k = \mat{I}_{d_k}$, and $\tgh^{(k)} = \tgk$.

Combing these two cases, the core tensor of TT decomposition can be given by $\tgk = \tgh^{(k)} \times_2 \mat{U}_k, k \in [K]$. Element-wisely, we can present  TT decomposition of $\tx$ as
\begin{equation}
	\begin{split}
			& X({i_1,\cdots, i_K}) \\
			 = &   \mg^{(1)}(:,i_1,:)  \mg^{(2)}(:,i_2,:) \cdots  \mg^{(K)}(:,i_K,:)  \\
			=&  \sum_{j_1=1}^{R_1} \cdots \sum_{j_K=1}^{R_K}  \mgh^{(1)} (:,j_1,:)  {U}_1(j_1, i_1) \\
			& \mgh^{(2)} (:,j_2,:)  {U}_2(j_2, i_2) \cdots \mgh^{(K)}(:,j_K,:)  {U}_K(j_K, i_K) \\
			=&  \sum_{j_1=1}^{R_1} \cdots \sum_{j_K=1}^{R_K}  \mgh^{(1)} (:,j_1,:) \mgh^{(2)} (:,j_2,:) \cdots  \\
			&   \mgh^{(K)}(:,j_K,:)   {U}_1(j_1, i_1)   {U}_2(j_2, i_2) \cdots  {U}_K(j_K, i_K) \\
			=& \sum_{j_1=1}^{R_1} \cdots \sum_{j_K=1}^{R_K} \mcode{TT}(\tgh^{(1)}, \cdots, \tgh^{(K)})(i_1, \cdots, i_K) \\ 
			&   {U}_1(j_1,i_1) \cdots  {U}_K(j_K,i_K)\\
		=& \mcode{TT}(\tgh^{(1)}, \cdots, \tgh^{(K)}) \times_1 \mat{U}_1(:,i_1)\cdots \times_K\mat{U}_K (:,i_K). \\
	\end{split} 
\end{equation}
Thus, we have $\tx = \txh \times_1 \mat{U}_1 \cdots \times_K \mat{U}_K$, where $\txh = \mcode{TT}(\tgh^{(1)}, \tgh^{(2)},\cdots,\tgh^{(K)} )$  and $\mat{U}_k \in St(d_k, R_k), k \in [K]$. 

Proof of Theorem \ref{theorem-rank-property} is completed.
\end{proof}

In the next theorem, we demonstrate that under a Tucker compression format, TTNN of the original tensor can be equivalently converted to that of a much smaller tensor.

\begin{theorem} \label{theorem-ttnn-property}
Let $\tx \in \matr^{d_1\times \cdots \times d_K}$ be a $K$th-order tensor with TT rank $[r_1, \cdots r_{K-1}]$, the TTNN of $\tx$ can be given by
	\begin{equation}
		\| \tx\|_{\emph{ttnn}} = \| \txh \|_{\emph{ttnn}}, 
	\end{equation}
where $\tx = \txh \times_{1} \mat{U}_1   \cdots  \times_K \mat{U}_K, \mat{U}_k \in St(d_k, R_k), k \in [K]$. 
\end{theorem}
\begin{proof}
	According to  Lemma 3 in \cite{Ee2014}, we have 
	\begin{equation} \label{eq-tt-unfold-xhat}
		\mxk = (\mat{U}_k \otimes \cdots \otimes \mat{U}_1) \mxhk (\mat{U}_K \otimes \cdots \otimes \mat{U}_{k+1} )^{\top}. 
	\end{equation}
Note that the Kronecker product of multiple orthogonal matrices is still orthogonal matrix, that is, 
\begin{equation} \label{eq-orthogonal-kronecker}
	\begin{split}
			 & (\mat{U}_k \otimes \cdots \otimes \mat{U}_1) ^{\top}  (\mat{U}_k \otimes \cdots \otimes \mat{U}_1)  \\
			 &  = ( \mat{U}_k^{\top} \mat{U}_k \otimes \cdots \otimes \mat{U}_1^{\top} \mat{U}_1 )   = \mat{I}_{M},
	\end{split}
\end{equation}
where $M=\prod_{j=1}^{k} {R_j} $. Combining Eq. (\ref{eq-orthogonal-kronecker}) and Eq. (\ref{eq-tt-unfold-xhat}), we have 
\begin{equation}
	\| \mxk   \|_* = \| \mxhk \|_*, k \in [K-1]. 
\end{equation}

Proof of Theorem \ref{theorem-ttnn-property} is completed.
\end{proof}

\begin{algorithm}[t] 
	\caption{Optimization Algorithm for FTTNN-based \mbox{TRPCA} Model}
	\begin{algorithmic}[1]
		\renewcommand{\algorithmicrequire}{\textbf{input}}
		\REQUIRE $\ty, \tau.$
		\renewcommand{\algorithmicrequire}{\textbf{initialize}}
		\REQUIRE $\mu=10^{-2}, \mu_{\text{max}}=10^{10}, \rho=1.1, tol=10^{-8}$, $\tgk$ using $\mathcal{N}(0,1)$ distribution.
		\WHILE {not converge}
		\STATE Update {\small $\tm^{k,t+1} = \mcode{fold}_k  \Big( \mathcal{D}_{\alpha_k / \mu } ( \mxhk^{t} - \frac{1}{\mu} \mat{Q}_{[k]}^{k,t}  ) \Big)$,} where $\mathcal{D}_{\alpha_k/\mu} (\cdot)$ denotes the SVT operator \cite{cai2010singular}. 
		\STATE Update {  \small  $\tx^{t+1}  = \ty^t - \ts^t + \txh^t \times_1 \mat{U}_1^t \cdots \times_K \mat{U}_K^t + \frac{1}{\mu} (\te^t - \tp^t ) .$ }
		\STATE  Update  { \small $\ts^{t+1} = \mathcal{S}_{\tau} ( \ty^t -\tx^{t+1} + \frac{1}{\mu} \te^{t} )$}, where $\mathcal{S}_{\tau}$ denotes the soft-shrinkage operator \cite{Beck2009}.  
		\STATE Update { \small   $	\txh^{t+1} = \frac{1 }{K\mu} \Big(  ( \mu \tx^{t+1} +\tp^{t}) \times_1 \mat{U}_1^{t,\top}\cdots \times_K \mat{U}_K^{t,\top} + \sum_{k=1}^{K} \mu \tensor{M}^{k,t} +   \tensor{Q}^{k,t} \Big ) $}.
		\STATE Update {\small $\mat{U}_k^{t+1} = \mat{A}_k \mat{B}_k$}, where {\small $( \frac{1}{\mu}\mat{P}_{(k)}^t + \mx_{(k)}^{t+1} )  \mat{V}_{(k)} = \mat{A }_k \mat{D}_k \mat{B}_k$}  is the SVD of {\small $( \frac{1}{\mu}\mat{P}_{(k)} + \mx_{(k)})  \mat{V}_{(k)}$}, and {\small $\tensor{V} = \tx^{t+1} \times_1 \mat{U}_1^{t+1,\top} \cdots \times_{k-1} \mat{U}_{k-1}^{t+1,\top} \times_{k+1} \mat{U}_{k+1}^{t,\top} \cdots \mat{U}_K^{t,\top}$. }
		\STATE Update {\small  $\tq^{k,t+1} = \tq^{k,t} + \mu ( \tm^{k,t+1} - \txh^{t+1}), k \in [K-1]$. }
		\STATE Update { \small   $\tp^{t+1} = \tp^{t} + \mu ( \tx^{t+1} - \txh^{t+1} \times_1 \mat{U}_1^{t+1} \times_2 \cdots  \times_K \mat{U}_K^{t+1} )$. } 
		\STATE Update $\mu  = \max(\rho \mu, \mu_{\text{max}})$.
		\STATE Check the convergence condition: 
		\STATE   $\max   \big ( \frac{\| \ts_{t+1}  - \ts_t \|_F}{\| \ts_t  \|_F}, \frac{\| \tx_{t+1}  - \tx_t \|_F}{\| \tx_t  \|_F}   \big) \leq tol$. 
		\STATE $t  \leftarrow t+1$.
		\ENDWHILE
	\end{algorithmic} 
	\label{alg-admm-tc}
\end{algorithm}

\subsection{Efficient TRPCA under Hybrid Model of Tucker and TT}
TRPCA aims to recover the low-rank and sparse components from their sum. The low-rank TT-based TRPCA model can be formulated as
\begin{equation} \label{eq-rtpca-ttnn}
\min_{\tx, \ts}   \| \tx \|_{\text{ttnn}} + \tau \| \ts\|_1\\
	\text{ s.t. } \ty = \tx + \ts, 
\end{equation}
where $\tau>0$ denotes the hyper-parameter. Combining Theorem \ref{theorem-rank-property} and Theorem \ref{theorem-ttnn-property},  problem (\ref{eq-rtpca-ttnn}) can be equivalently given in a fast TTNN (FTTNN) minimization format:
\begin{equation} \label{eq-ttnn-rpca-smaller-size}
	\begin{split}
			& \min_{\tx, \txh, \ts, \{\mat{U}_k \}_{k=1}^K }   \| \txh \|_{\text{ttnn}} + \tau \| \ts\|_1\\
	\text{s.t. } &\tx = \txh \times_1 \mat{U}_1 \times_2 \cdots \times_K \mat{U}_K, \\
	&  \ty = \tx+  \ts,  \mat{U}_k \in St(I_k, R_k).
	\end{split}
\end{equation}
By incorporating auxiliary variables $\{ \tmk \}_{k=1}^K$,  the augmented Lagrangian function of problem (\ref{eq-ttnn-rpca-smaller-size}) is given by
\begin{equation} \label{eq-augmented-lagrangian-function}
		\begin{split}
& L_{\mu}(  \{ \tmk\}_{k=1}^{K-1}, \txh, \ts, \{ \mat{U}_k \}_{k=1}^{K}, \{ \tqk\}_{k=1}^{K-1}, \te, \tp )  \\
= &  \sum_{k=1}^{K-1} \alpha_k \| \mmk \|_{*} + \tau \| \ts\|_1  \\
		 & +  \sum_{k=1}^{K-1} \langle \tqk, \tmk - \txh \rangle + \frac{\mu}{2} \| \tmk - \txh \|_F^2 \\
		 & + \langle \ty - \tx - \ts, \te\rangle  +  \frac{\mu}{2} \| \ty - \tx - \ts  \|_F^2 \\
 		& + \langle \tp,  \tx - \txh \times_1 \mat{U}_1 \times_2 \cdots \times_K \mat{U}_K \rangle  \\
		 &  + \frac{\mu}{2} \| \tx - \txh \times_1 \mat{U}_1 \times_2 \cdots \times_K \mat{U}_K  \|_F^2 \\
		\text{s.t. } & \mat{U}_k  \in St(I_k, R_k), k \in [K],
	\end{split}
\end{equation}
where $\mu > 0$ denotes penalty parameter, $\{\tqk \}_{k=1}^{K}, \te$  and $\tp$ are Lagrange multipliers. All   variables of Eq. (\ref{eq-augmented-lagrangian-function}) can be solved separately based on ADMM method \cite{Boyd2010a}. The update details are summarized in Algorithm \ref{alg-admm-tc}.

Compared with the time complexity $\mathcal{O}(d^{3K/2})$ in  TTNN of the original big tensor, the proposed FTTNN-based TRPCA model only requires time complexity $\mathcal{O}(KR^2d^{K-1} + KR^{3K/2}+KRd^K)$, which will  greatly accelerate the optimization algorithm if the given rank $R$ is significantly low.

\begin{table}[]
	\centering
	\caption{Recovery Results of the Proposed FTTNN and TTNN on Synthetic Tensor  Data}
	\label{table-synthetic-data}
	\resizebox{\linewidth}{27mm}{
	\begin{tabular}{c|c|c|cccc} \hline  \hline 
		d                   & TT rank               & $\text{NR}$                    & Alg &  {RSE}-$\tx$ & {RSE}-$\ts$ & Time \\ \hline 
		\multirow{8}{*}{30} & \multirow{4}{*}{3} & \multirow{2}{*}{5\%} & TTNN  & 1.52e-8  &  1.36e-10     & 18.05    \\
		&                    &                       & FTTNN (ours) &    \textbf{1.83e-9}     & \textbf{3.63e-11}      &  \textbf{7.29}      \\  \cline{3-7}
		&                    & \multirow{2}{*}{10\%}  & TTNN  & 1.54e-8     &   1.30e-10   & 20.34     \\
		&                    &                       & FTTNN (ours) & \textbf{1.41e-9}     &  \textbf{2.31e-11}     &  \textbf{7.81}     \\ \cline{2-7}
		& \multirow{4}{*}{4} & \multirow{2}{*}{5\%} & TTNN  & 1.52e-8     & 2.09e-10     &  17.64     \\
		&                    &                       & FTTNN (ours) &  \textbf{1.52e-9}    & \textbf{4.72e-11}      &  \textbf{8.29}      \\\cline{3-7}
		&                    & \multirow{2}{*}{10\%}  & TTNN  & 1.49e-8     & 2.00e-10     & 19.03     \\
		&                    &                       & FTTNN (ours) &   \textbf{1.06e-9}   &  \textbf{2.69e-11}     &  \textbf{8.47}     \\\cline{1-7}
		\multirow{8}{*}{40} & \multirow{4}{*}{4} & \multirow{2}{*}{5\%} & TTNN  &  1.46e-8     &   1.07e-10    &  129.69     \\
		&                    &                       & FTTNN (ours) &  \textbf{1.89e-9}    & \textbf{3.40e-11}     &  \textbf{23.27}     \\ \cline{3-7}
		&                    & \multirow{2}{*}{10\%}  & TTNN  &  1.31e-8    & 9.64e-11    & 122.99    \\
		&                    &                       & FTTNN (ours) & \textbf{1.26e-9}     & \textbf{1.77e-11}      & \textbf{24.46}     \\ \cline{2-7}
		& \multirow{4}{*}{5} & \multirow{2}{*}{5\%} & TTNN  & 1.31e-8     &  1.40e-10    &   112.36   \\
		&                    &                       & FTTNN (ours) &  \textbf{1.45e-9}    & \textbf{3.60e-11}      & \textbf{28.58}      \\ \cline{3-7}
		&                    & \multirow{2}{*}{10\%}  & TTNN  & \textbf{1.32e-8}      &  \textbf{1.34e-10}    &117.01      \\
		&                    &                       & FTTNN (ours) &  5.46e-7  & 1.25e-8    & \textbf{27.72}     \\
		 \hline  \hline 
	\end{tabular} }
\end{table}

\section{Numerical  Experiments}
In this section, we present  the numerical experiment results on synthetic tensor data  as well as color image and video data, of the proposed and state-of-the-art model, namely BRTF  \cite{Zhao2016}, SNN  \cite{Goldfarb2014}, tSVD \cite{zhang2014novel} and TTNN \cite{Yang2020a}. 
All experiments are tested with respect to different sparse noise ratios (NR), which is  given by $\text{NR}=N/\prod_{k=1}^{K} d_k \times 100\%$, where $N$ denotes the number of sparse component entries.
The relative standard error (RSE) is adopted as a performance metric, and is given by $\text{RSE} = \| \tx^* - \tx_0\|_F / \| \tx_0\|_F$, where $\tx^*$ and $\tx_0$ are the estimated and true tensor, respectively. 
Matlab implementation of the proposed method is publicly available \footnote{https://github.com/ynqiu/fast-TTRPCA}. 

\subsection{Synthetic Data}
 We generate a low-rank tensor $\tx_0 \in \matr^{d \times d \times d\times d}$ by TT contraction  \cite{Wen2008} with TT rank $[r, r, r]$. The entries of each core tensor are generated by i.i.d. standard Gaussian distribution, i.e.,  $ g_{v_k, i_k, v_{k+1}}^{(k)} \sim   \mathcal{N}(0, 1), k \in [K]$. The support of sparse noise $\ts_0$ is uniformly sampled at random. For $(i_1, i_2, i_3, i_4) \in \text{supp}(\ts_0)$, we let $\ts_0(i_1, i_2, i_3, i_4) = \tb(i_1, i_2, i_3, i_4)$, where $\tb$ is generated by the independent  Bernoulli distribution. The observed tensor is formed by $\ty = \tx_0 + \ts_0$. The parameter $\tau =1/(K-1) \sum_{k=1}^{K-1} 1/ \sqrt{ \max ( d_{1,k}, d_{2,k} )} $. For the weight $\alpha_k, k \in [K-1],$ we adopt the same strategy used in \cite{Bengua2017a}.

\subsubsection{Effectiveness of the Proposed FTTNN-based TRPCA} To verify  the effectiveness of the proposed FTTNN-based  TRPCA, we conduct   experiments on multiple conditions. We let $d \in \{30, 40\}$, and $r\in \{ 3,4,5\}$. The sparse noise ratio is selected in a candidate set: $ \text{NR} \in \{ 5\%,10\%\}$. The given rank is set to $R_1 = R_4 = \mcode{round}(1.2r)$ and $R_2 = R_3 =  \mcode{round}(1.2r^2) $. For each fixed setting, we repeat  the experiment 10 times and report their average. Table \ref{table-synthetic-data} shows the  results of FTTNN and TTNN on synthetic tensor data.  As can be seen, FTTNN provides  lower RSE on both  low-rank $\tx$ and  sparse $\ts$ components compared with TTNN in most cases. Moreover, FTTNN  is at least 2 times faster than TTNN when $d=30$ and at least 4 times faster than TTNN when $d=40$.

\begin{figure}[t]
	\centering
	\subfigure[]{
	\begin{minipage}[t]{0.325\linewidth}
			\centering
		\includegraphics[width=\linewidth]{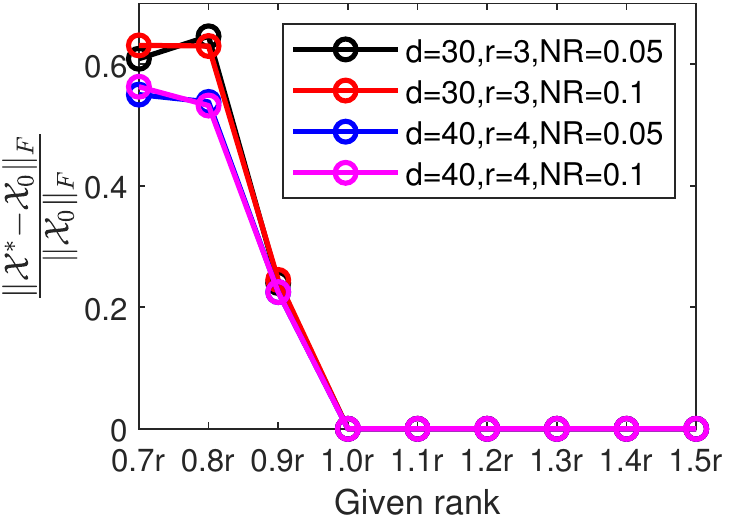}
	\end{minipage}} 
\hspace{-1.2em}
\subfigure[]{
	\begin{minipage}[t]{0.325\linewidth}
	\centering
	\includegraphics[width=\linewidth]{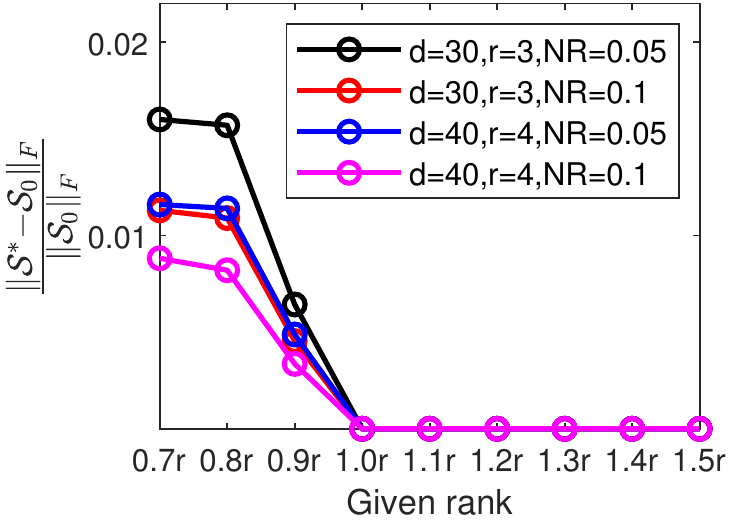}
\end{minipage}}
\hspace{-1.2em}
\subfigure[]{
	\begin{minipage}[t]{0.325\linewidth}
	\centering
	\includegraphics[width=\linewidth]{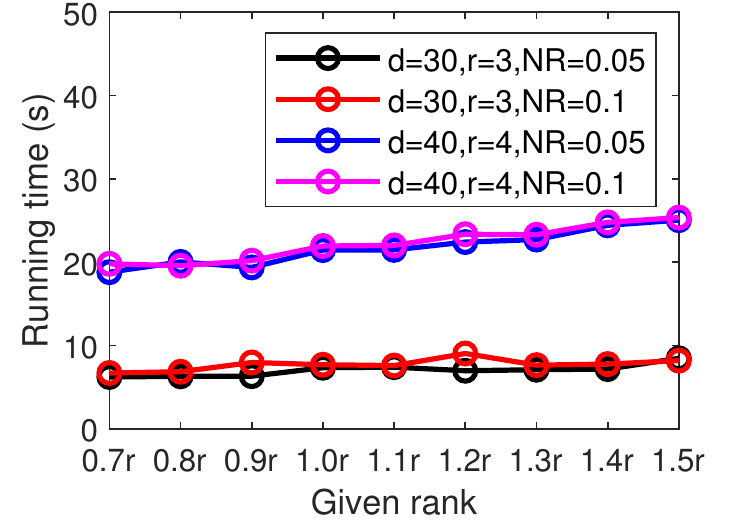}
\end{minipage} }
\caption{Plots of RSE  of recovered $\tx^*$, $\ts^*$ and running time versus varying initialization of rank. (a) RSE-$\tx$ versus different ranks; (b) RSE-$\ts$ versus different ranks; (c) Running time versus different ranks.}
\label{fig-plot-synthetic-data}
\end{figure}

\subsubsection{Robustness of the Given Rank} In this part, we investigate robustness of the given rank for FTTNN. Similar with the above simulation, we set $d \in \{ 30, 40\}$, $r \in \{ 3, 4\}$ and let the sparse noise ratio  $\text{NR}\in \{ 5\%, 10\%\}$. The  given rank of the proposed model is set to $R_1 = R_4 = \mcode{round}(qr)$ and $R_2 = R_3 =  \mcode{round}(qr^2) $ with $q \in \{0.7,0.8, \cdots, 1.5\}$. In Fig. \ref{fig-plot-synthetic-data}, we plot its average RSE and running time versus different given ranks. It can be observed that the proposed FTTNN achieves stable recovery performance versus different given ranks if $q\geq 1$, which verifies the correctness of Theorem \ref{theorem-ttnn-property}. Additionally, running time of the proposed FTTNN grows slowly as the given rank increases. 


 \begin{figure}[t]
	\centering
	\includegraphics[width=0.9\linewidth]{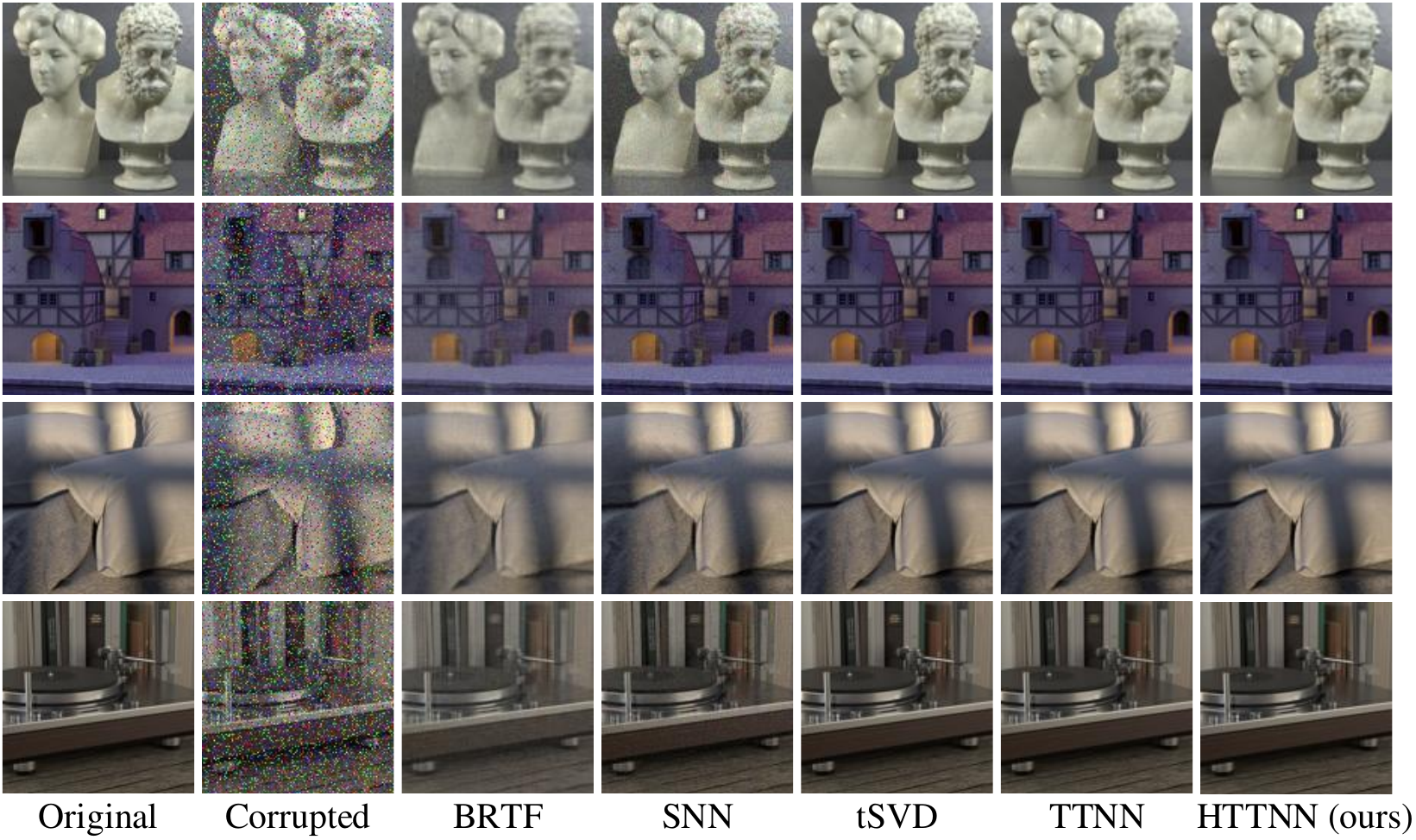}
	\caption{Visual performance of the proposed FTTNN compared with related BRTF, SNN, tSVD and TTNN (best seen in larger resolution on monitor). }
	\label{fig-light-field-frame}
\end{figure}

\begin{table}[t]
	\centering
	\caption{Average  Recovery Results of Compared Models on Four Benchmark Images}
	\resizebox{0.9\linewidth}{11mm}{
	\begin{tabular}{l | c | c |c | c |c} \hline \hline 
		Models &  greek   & medieval2  & pillows & vinyl &   Avg. time     \\ \hline  
  		Noisy   & 0.2753   &   0.5175 & 0.3364  & 0.5854 &   -   \\
		BRTF   & 0.0934    &   0.1666  & 0.0941 & 0.1965 & 112.75 \\
		SNN    & 0.0733  &   0.0859 & 0.0535  & 0.0852 & 930.04 \\
		tSVD    & 0.0516 &    0.0497 & 0.0383  & 0.0524  & 71.67 \\
		TTNN   & 0.0447  &  0.0403  & 0.0310   & 0.0333 & 104.68  \\
		FTTNN (ours)  & \textbf{0.0354}  &  \textbf{0.0348}  &   \textbf{0.0207}  & \textbf{0.0286} &  \textbf{53.13}  \\
		\hline \hline 
	\end{tabular} } 
	\label{table-light-field-images}
\end{table}

\subsection{Robust Recovery of Noisy Light Field Images}
In this part, we conduct robust noisy light field image recovery experiment on four light field benchmark images\footnote{https://lightfield-analysis.uni-konstanz.de/}, namely,  ``greek",``medieval2",``pillows" and ``vinyl". The dimensions of each image is down sampled to $128\times 128 \times 3 \times 81$. The given rank of the proposed model is set to $[80,80,3,10]$ and the weight $\alpha$ is set to $[0.1,0.8,0.1]$. For each image, we randomly select $20\%$ entries with their values being randomly distributed in $[0,255]$.  Fig. \ref{fig-light-field-frame} depicts the $50$th recovered image, i.e., $\ty(:,:,:,50)$. From Fig. \ref{fig-light-field-frame}, we can observe that the results obtained by the proposed FTTNN are superior to the compared models, especially for the recovery of local details. 
Table \ref{table-light-field-images} shows the recovered RSE and average running time on four benchmark images.  Compared with the state-of-the-art models, the proposed FTTNN achieves both minimum RSE and average running time in all light field images, which indicates its efficiency. 

\begin{figure}
	\centering
	\subfigure[]{
		\begin{minipage}[t]{0.42\linewidth}
			\centering
			\includegraphics[width=\linewidth]{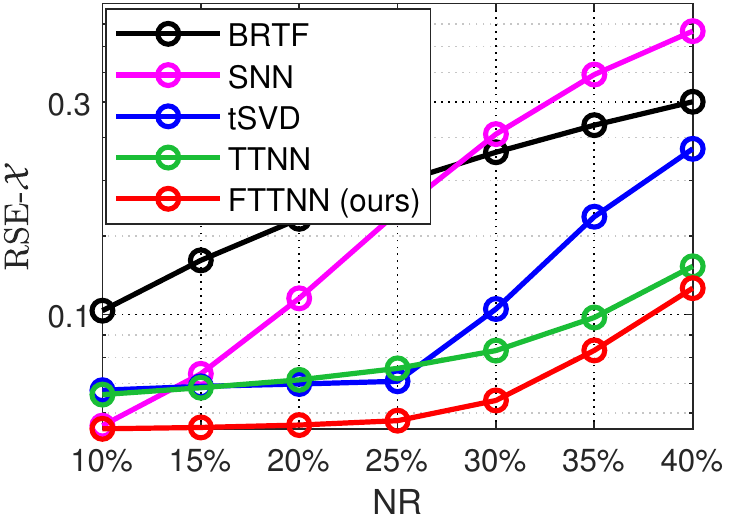}
	\end{minipage}}
	\subfigure[]{
		\begin{minipage}[t]{0.42\linewidth}
			\centering
				\includegraphics[width=\linewidth]{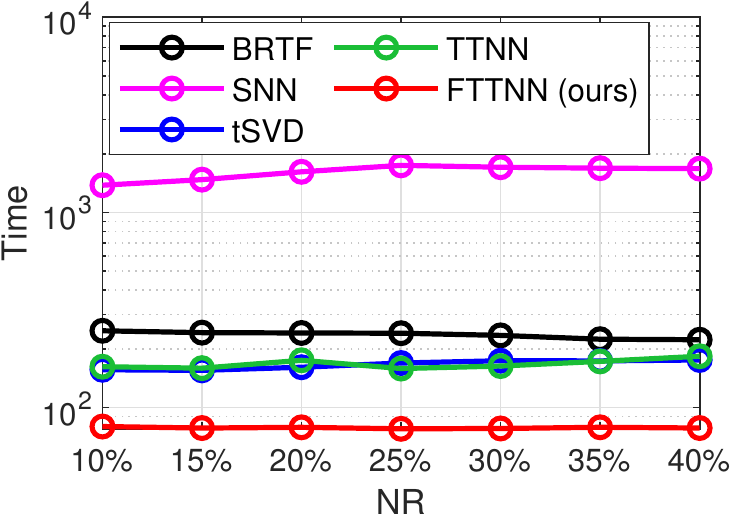}
	\end{minipage} }
	\caption{Plots of average RSE and running time of compared models under different NR. (a) RSE versus   NR; (b) running time versus   NR.}
	\label{fig-plot-NR-data}
\end{figure}

\subsection{Robust Recovery of Noisy Video Sequences}
In this part, we investigate the robust recovery performance of the compared models on five YUV video sequences\footnote{http://trace.eas.asu.edu/yuv/}, namely, ``akiyo",``bridge",``grandma", ``hall" and ``news". We select the first 100 frames of videos. Thus, the video data are the fourth-order tensors of size $144\times 176 \times 3\times 100$. The given rank and weight are set to the same as the above section.
For each video, we set the noise ratio $\text{NR} \in \{10\%, \cdots, 40\% \}$. 
Fig. \ref{fig-plot-NR-data} shows the average RSE of five videos versus different NR. The proposed FTTNN obtains the lowest average RSE on five videos compared with state-of-the-art models.
Additionally, our model is the fastest models, and is at least two times faster than TTNN. Since SNN, TTNN and tSVD have to compute SVD or tSVD on the original large-scale video data, they run slower.

\section{Conclusion}
In this letter, an efficient TRPCA model is proposed based on low-rank TT. By investigating the relationship between Tucker decomposition and TT decomposition, TTNN of the original big tensor is proved to be equivalent to that of a much smaller tensor under a Tucker compression format, thus reducing the computational cost of SVD operation. 
Experimental results show that the proposed model outperforms the state-of-the-art models in terms of RSE and running time. 

\bibliographystyle{IEEEbib}
\bibliography{FastTTRPCA}

\begin{thebibliography}{10}

\bibitem{TamaraG.KoldaBrettW.Bader}
Tamara~G. Kolda and Brett~W. Bader,
\newblock ``{Tensor decompositions and applications},''
\newblock {\em SIAM Review}, vol. 51, no. 3, pp. 455--500, 2009.

\bibitem{spm-cichocki-2015}
Andrzej Cichocki, Danilo Mandic, Lieven {De Lathauwer}, Guoxu Zhou, Qibin Zhao,
  Cesar Caiafa, and Huy~Anh Phan,
\newblock ``{Tensor decompositions for signal processing applications: From
  two-way to multiway component analysis},''
\newblock {\em IEEE Signal Processing Magazine}, vol. 32, no. 2, pp. 145--163,
  2015.

\bibitem{zhang2017low}
Jing Zhang, Xinhui Li, Peiguang Jing, Jing Liu, and Yuting Su,
\newblock ``{Low-rank regularized heterogeneous tensor decomposition for
  subspace clustering},''
\newblock {\em IEEE Signal Processing Letters}, vol. 25, no. 3, pp. 333--337,
  2017.

\bibitem{inoue2009robust}
Kohei Inoue, Kenji Hara, and Kiichi Urahama,
\newblock ``{Robust multilinear principal component analysis},''
\newblock in {\em 2009 IEEE 12th International Conference on Computer Vision}.
  IEEE, 2009, pp. 591--597.

\bibitem{Zare2018}
Ali Zare, Alp Ozdemir, Mark~A. Iwen, and Selin Aviyente,
\newblock ``{Extension of PCA to higher order data structures: An introduction
  to tensors, tensor decompositions, and tensor PCA},''
\newblock {\em Proceedings of the IEEE}, vol. 106, no. 8, pp. 1341--1358, 2018.

\bibitem{zhang2014novel}
Zemin Zhang, Gregory Ely, Shuchin Aeron, Ning Hao, and Misha Kilmer,
\newblock ``{Novel methods for multilinear data completion and de-noising based
  on tensor-SVD},''
\newblock in {\em Proceedings of the IEEE conference on computer vision and
  pattern recognition}, 2014, pp. 3842--3849.

\bibitem{Goldfarb2014}
Donald Goldfarb and Zhiwei Qin,
\newblock ``{Robust low-rank tensor recovery: Models and algorithms},''
\newblock {\em SIAM Journal on Matrix Analysis and Applications}, vol. 35, no.
  1, pp. 225--253, 2014.

\bibitem{Ren2017}
Jineng Ren, Xingguo Li, and Jarvis Haupt,
\newblock ``{Robust PCA via tensor outlier pursuit},''
\newblock {\em Conference Record - Asilomar Conference on Signals, Systems and
  Computers}, pp. 1744--1749, 2017.

\bibitem{Bahri2019a}
Mehdi Bahri, Yannis Panagakis, and Stefanos Zafeiriou,
\newblock ``{Robust Kronecker Component Analysis},''
\newblock {\em IEEE Transactions on Pattern Analysis and Machine Intelligence},
  vol. 41, no. 10, pp. 2365--2379, 2019.

\bibitem{Gu2014a}
Quanquan Gu, Huan Gui, and Jiawei Han,
\newblock ``{Robust tensor decomposition with gross corruption},''
\newblock {\em Advances in Neural Information Processing Systems}, vol. 2, no.
  January, pp. 1422--1430, 2014.

\bibitem{Zhang2016a}
Miaohua Zhang, Yongsheng Gao, Changming Sun, John {La Salle}, and Junli Liang,
\newblock ``{Robust tensor factorization using maximum correntropy
  criterion},''
\newblock {\em Proceedings - International Conference on Pattern Recognition},
  vol. 0, no. 1, pp. 4184--4189, 2016.

\bibitem{li2013robust}
Qun Li, Xiangqiong Shi, and Dan Schonfeld,
\newblock ``{Robust HOSVD-based higher-order data indexing and retrieval},''
\newblock {\em IEEE Signal Processing Letters}, vol. 20, no. 10, pp. 984--987,
  2013.

\bibitem{10.1145/2512329}
Christopher~J Hillar and Lek-Heng Lim,
\newblock ``{Most Tensor Problems Are NP-Hard},''
\newblock {\em J. ACM}, vol. 60, no. 6, 2013.

\bibitem{Friedland2017}
Shmuel Friedland and Lek-Heng Lim,
\newblock ``{Nuclear norm of higher-order tensors},''
\newblock {\em Mathematics of Computation}, vol. 87, no. 311, pp. 1255--1281,
  sep 2017.

\bibitem{Zhao2016}
Qibin Zhao, Guoxu Zhou, Liqing Zhang, Andrzej Cichocki, and Shun~Ichi Amari,
\newblock ``{Bayesian Robust Tensor Factorization for Incomplete Multiway
  Data},''
\newblock {\em IEEE Transactions on Neural Networks and Learning Systems}, vol.
  27, no. 4, pp. 736--748, 2016.

\bibitem{tucker1966some}
Ledyard~R Tucker,
\newblock ``{Some mathematical notes on three-mode factor analysis},''
\newblock {\em Psychometrika}, vol. 31, no. 3, pp. 279--311, 1966.

\bibitem{JiLiu2010b}
Ji~Liu, Przemyslaw Musialski, Peter Wonka, and Jieping Ye,
\newblock ``{Tensor completion for estimating missing values in visual data},''
\newblock {\em IEEE Transactions on Pattern Analysis and Machine Intelligence},
  vol. 35, no. 1, pp. 208--220, jan 2013.

\bibitem{huang2015provable}
Bo~Huang, Cun Mu, Donald Goldfarb, and John Wright,
\newblock ``{Provable models for robust low-rank tensor completion},''
\newblock {\em Pacific Journal of Optimization}, vol. 11, no. 2, pp. 339--364,
  2015.

\bibitem{Bengua2017a}
Johann~A. Bengua, Ho~N. Phien, Hoang~Duong Tuan, and Minh~N. Do,
\newblock ``{Efficient Tensor Completion for Color Image and Video Recovery:
  Low-Rank Tensor Train},''
\newblock {\em IEEE Transactions on Image Processing}, vol. 26, no. 5, pp.
  2466--2479, 2017.

\bibitem{9200561}
Xiao Gong, Wei Chen, Jie Chen, and Bo~Ai,
\newblock ``{Tensor denoising using low-rank tensor train decomposition},''
\newblock {\em IEEE Signal Processing Letters}, vol. 27, pp. 1685--1689, 2020.

\bibitem{Yang2020a}
Jing~Hua Yang, Xi~Le Zhao, Teng~Yu Ji, Tian~Hui Ma, and Ting~Zhu Huang,
\newblock ``{Low-rank tensor train for tensor robust principal component
  analysis},''
\newblock {\em Applied Mathematics and Computation}, vol. 367, 2020.

\bibitem{Wen2008}
I.~V. Oseledets,
\newblock ``{Tensor-train decomposition},''
\newblock {\em SIAM Journal on Scientific Computing}, vol. 33, no. 5, pp.
  2295--2317, jan 2011.

\bibitem{Ee2014}
Cun Mu, Bo~Huang, John Wright, and Donald Goldfarb,
\newblock ``{Square deal: Lower bounds and improved relaxations for tensor
  recovery},''
\newblock {\em 31st International Conference on Machine Learning, ICML 2014},
  vol. 2, pp. 1242--1250, 2014.

\bibitem{cai2010singular}
Jian-Feng Cai, Emmanuel~J Cand{\`{e}}s, and Zuowei Shen,
\newblock ``{A singular value thresholding algorithm for matrix completion},''
\newblock {\em SIAM Journal on optimization}, vol. 20, no. 4, pp. 1956--1982,
  2010.

\bibitem{Beck2009}
Amir Beck and Marc Teboulle,
\newblock ``{A Fast Iterative Shrinkage-Thresholding Algorithm for Linear
  Inverse Problems},''
\newblock {\em SIAM Journal on Imaging Sciences}, vol. 2, no. 1, pp. 183--202,
  jan 2009.

\bibitem{Boyd2010a}
Stephen Boyd, Neal Parikh, Eric Chu, Borja Peleato, and Jonathan Eckstein,
\newblock ``{Distributed optimization and statistical learning via the
  alternating direction method of multipliers},''
\newblock {\em Foundations and Trends in Machine Learning}, vol. 3, no. 1, pp.
  1--122, 2010.

\end{thebibliography}

\end{document}